\documentclass{article}

\usepackage{colt10e}

\newcommand{\fullpaper}{}
\newcommand{\conferencepaper}[1]{}

\usepackage{color}
\usepackage{amsmath}
\usepackage{amssymb}
\usepackage{graphicx} 
\usepackage{wrapfig}

\usepackage[numbers]{natbib}

\usepackage{fancybox,graphicx}
\usepackage{float}

\newenvironment{algorithm}[1][\  ] %
{
\rm
\begin{tabbing} 
....\=.....\=.....\=.....\=.....\=  \+ \kill
} %
{\end{tabbing}
}

\floatstyle{ruled}
\newfloat{Algorithm}{H}{alg}

{
\begin{minipage}{1.0\linewidth} \begin{algorithm} %
}
{ \end{algorithm} \end{minipage} 
}

\usepackage{rotating}
\usepackage{array}
\usepackage{multirow}

\usepackage{tikz}
\usetikzlibrary{decorations.pathmorphing,backgrounds,fit}
\usepgfmodule{decorations}

\newtheorem{lemma}{Lemma}
\newtheorem{corollary}{Corollary}
\newtheorem{theorem}{Theorem}

\newtheorem{assumption}{Assumption}

\newcommand{\x}{{\mathbf x}}

\newcommand{\z}{{\mathbf z}}
\newcommand{\w}{{\mathbf w}}

\renewcommand{\v}{{\mathbf v}}

\newcommand{\sgn}{{\mathrm{sgn}}}

\newcommand{\BlackBox}{\rule{1.5ex}{1.5ex}}  
\newenvironment{proof}{\par\noindent{\bf Proof\ }}{\hfill\BlackBox\\[2mm]}

\newcommand{\reals}{\mathbb{R}}

\newcommand{\D}{\mathcal{D}}
\newcommand{\X}{\mathcal{X}}

\newcommand{\V}{\mathbb{V}}
\newcommand{\N}{\mathbb{N}}
\DeclareMathOperator*{\E}{\mathbb{E}}
\DeclareMathOperator*{\prob}{\mathbb{P}}

\newcommand{\inner}[1]{\langle #1 \rangle}

\newcommand{\thalf}{\tfrac{1}{2}}
\newcommand{\indct}[1]{\mathbf{1}(#1)}
\newcommand{\err}{\mathrm{err}}
\newcommand{\eqdef}{\stackrel{\mathrm{def}}{=}}
\newcommand{\erf}{\mathrm{erf}}
\newcommand{\sig}{\mathrm{sig}}
\newcommand{\pw}{\mathrm{pw}}
\newcommand{\poly}{\mathrm{poly}}

\DeclareMathOperator*{\argmin}{argmin} 

\renewcommand{\eqref}[1]{Equation~(\ref{#1})}
\newcommand{\figref}[1]{Figure~\ref{#1}}
\newcommand{\secref}[1]{Section~\ref{#1}}
\newcommand{\appref}[1]{Appendix~\ref{#1}}
\newcommand{\thmref}[1]{Theorem~\ref{#1}}
\newcommand{\lemref}[1]{Lemma~\ref{#1}}

\title{
Learning Kernel-Based Halfspaces with the Zero-One Loss
}
\author{
Shai Shalev-Shwartz\\
The Hebrew University\\
\texttt{shais@cs.huji.ac.il}
\And
Ohad Shamir\\
The Hebrew University\\
\texttt{ohadsh@cs.huji.ac.il}
 \And
Karthik Sridharan \\
Toyota Technological Institute\\
\texttt{karthik@tti-c.org}
}

\begin{document}

\maketitle

\begin{abstract} 

  We describe and analyze a new algorithm for agnostically learning
  kernel-based halfspaces with respect to the \emph{zero-one} loss
  function. Unlike most previous formulations which rely on surrogate convex
  loss functions (e.g. hinge-loss in SVM and log-loss in logistic regression),
  we provide finite time/sample guarantees with respect to the more natural
  zero-one loss function. The proposed algorithm can learn kernel-based
  halfspaces in worst-case time $\poly(\exp(L\log(L/\epsilon)))$, for
  $\emph{any}$ distribution, where $L$ is a Lipschitz constant (which can be
  thought of as the reciprocal of the margin), and the learned classifier is
  worse than the optimal halfspace by at most $\epsilon$. We also prove a
  hardness result, showing that under a certain cryptographic assumption, no
  algorithm can learn kernel-based halfspaces in time polynomial in $L$.
 \end{abstract}

\section{Introduction}

A highly important hypothesis class in machine learning theory and applications
is that of halfspaces in a Reproducing Kernel Hilbert Space (RKHS). Choosing a 
halfspace based on empirical data is often performed using Support Vector 
Machines (SVMs) \cite{Vapnik98}.  SVMs replace the more natural 0-1 loss 
function with a convex surrogate -- the hinge-loss. By doing so, we can rely 
on convex optimization tools. However, there are no guarantees on how well the 
hinge-loss approximates the 0-1 loss function. There do exist some recent 
results on the \emph{asymptotic} relationship between surrogate convex loss 
functions and the 0-1 loss function \citep{Zhang04a,BartlettJoMc06}, but these 
do not come with finite-sample or finite-time guarantees. In this paper, we 
tackle the task of learning kernel-based halfspaces with respect to the 
non-convex 0-1 loss function. Our goal is to derive learning algorithms and to 
analyze them in the finite-sample finite-time setting.

 Following the standard statistical learning framework, we assume that there is an unknown distribution, $\D$, over the set of labeled examples, $\X \times \{0,1\}$, and our primary goal is to find a classifier, $h : \X \to \{0,1\}$, with low generalization error, 
\begin{equation} \label{eqn:def_err}
\err_\D(h) ~\eqdef 
\E_{(\x,y)\sim \D}[ |h(\x)-y|] ~.
\end{equation}
The learning algorithm is allowed to sample a training set of labeled examples,
$ (\x_1,y_1),\ldots,(\x_m,y_m)$, where each example is sampled i.i.d. from
$\D$, and it returns a classifier.  Following the agnostic PAC learning
framework \cite{KearnsScSe92}, we say that an algorithm
$(\epsilon,\delta)$-learns a concept class $H$ of classifiers using $m$
examples, if with probability of at least $1-\delta$ over a random choice of
$m$ examples the algorithm returns a classifier $\hat{h}$ that satisfies
\begin{equation} \label{eqn:aPAC}
\err_\D(\hat{h}) ~\le~ \inf_{h \in H} \err_\D(h) + \epsilon ~.
\end{equation}
We note that $\hat{h}$ does not necessarily belong to $H$. Namely, we
are concerned with \emph{improper} learning, which is as useful as
proper learning for the purpose of deriving good classifiers.  A
common learning paradigm is the Empirical Risk Minimization (ERM)
rule, which returns a classifier that minimizes the average error over
the training set,
\[
\hat{h} \in \argmin_{h \in H} \frac{1}{m} \sum_{i=1}^m |h(\x_i)-y_i| ~.
\]

The class of (origin centered) halfspaces is defined as follows. Let $\X$ be a
compact subset of a RKHS, which
w.l.o.g. will be taken to be the unit ball around the origin. Let $\phi_{0-1} :
\reals \to \reals$ be the function $\phi_{0-1}(a) = \indct{a \ge 0} = \thalf
(\sgn(a)+1)$. The class of halfspaces is the set of classifiers
\[
H_{\phi_{0-1}} ~\eqdef~ \{ \x \mapsto \phi_{0-1}(\inner{\w,\x}) \,:\, \w \in \X \} ~.
\]
Although we represent the halfspace using $\w \in \X$, which is a vector in the
RKHS whose dimensionality can be infinite, in practice we only need a function
that implements inner products in the RKHS (a.k.a. a kernel function), and one
can define $\w$ as the coefficients of a linear combination of examples in
our training set. To simplify the notation throughout the paper, we represent
$\w$ simply as a vector in the RKHS.

It is well known that if the dimensionality of $\X$ is $n$, then the VC
dimension of $H_{\phi_{0-1}}$ equals $n$. This implies that the number of
training examples required to obtain a guarantee of the form given in
\eqref{eqn:aPAC} for the class of halfspaces scales at least linearly with the
dimension $n$ \citep{Vapnik98}. Since kernel-based learning algorithms allow
$\X$ to be an infinite dimensional inner product space, we must use a different
class in order to obtain a guarantee of the form given in \eqref{eqn:aPAC}.

One way to define a slightly different concept class is to approximate the non-continuous function, $\phi_{0-1}$, with a Lipschitz continuous function, $\phi : \reals \to [0,1]$, which is often called a transfer function. For example, we can use a sigmoidal transfer function
\begin{equation} \label{eqn:sigdef}
\phi_{\sig}(a) ~\eqdef~ \frac{1}{1+ \exp(-4L\,a)} ~,
\end{equation}
which is a $L$-Lipschitz function. Other $L$-Lipschitz transfer functions are the erf function and the piece-wise linear function:
\begin{equation}\label{eq:erfpw}
\phi_{\erf}(a) ~\eqdef~ \thalf\left(1+\text{erf}\left(\sqrt{\pi} \,L\,a\right)\right)
~~~~,~~~~
\phi_{\pw}(a) ~\eqdef~ \max\left\{\min\left\{\tfrac{1}{2} + L\,a \,,\, 1\right\} \, 0\right\} 
\end{equation}
An illustration of these transfer functions is given in \figref{fig:erf}. 
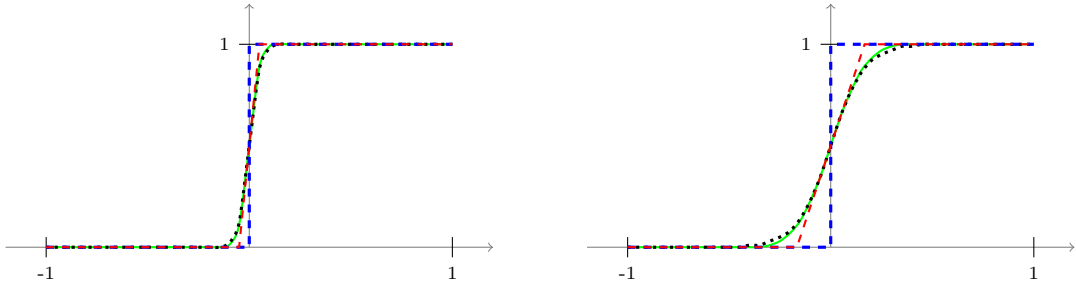
\begin{figure}[t]
\begin{center}
\begin{tikzpicture}[scale=2.7]

 \draw[->,gray] (-1.2,0) -- (1.2,0); 
 \draw[->,gray] (0,0) -- (0,1.2);

 \draw (-1,0.05) -- (-1,-0.05) node[below] {\scriptsize -1};
 \draw (1,0.05) -- (1,-0.05) node[below] {\scriptsize 1};
 \draw (0.05,1) -- (-0.05,1) node[left] {\scriptsize 1};

\draw[-, thick,green] plot[smooth] coordinates{(-1.00,0.00) (-0.95,0.00) (-0.90,0.00) (-0.85,0.00) (-0.80,0.00) (-0.75,0.00) (-0.70,0.00) (-0.65,0.00) (-0.60,0.00) (-0.55,0.00) (-0.50,0.00) (-0.45,0.00) (-0.40,0.00) (-0.35,0.00) (-0.30,0.00) (-0.25,0.00) (-0.20,0.00) (-0.15,0.00) (-0.10,0.01) (-0.05,0.11) (0.00,0.50) (0.05,0.89) (0.10,0.99) (0.15,1.00) (0.20,1.00) (0.25,1.00) (0.30,1.00) (0.35,1.00) (0.40,1.00) (0.45,1.00) (0.50,1.00) (0.55,1.00) (0.60,1.00) (0.65,1.00) (0.70,1.00) (0.75,1.00) (0.80,1.00) (0.85,1.00) (0.90,1.00) (0.95,1.00) (1.00,1.00) };

\draw[dotted,very thick,black] plot[smooth] coordinates{(-1.00,0.00) (-0.95,0.00) (-0.90,0.00) (-0.85,0.00) (-0.80,0.00) (-0.75,0.00) (-0.70,0.00) (-0.65,0.00) (-0.60,0.00) (-0.55,0.00) (-0.50,0.00) (-0.45,0.00) (-0.40,0.00) (-0.35,0.00) (-0.30,0.00) (-0.25,0.00) (-0.20,0.00) (-0.15,0.00) (-0.10,0.02) (-0.05,0.12) (0.00,0.50) (0.05,0.88) (0.10,0.98) (0.15,1.00) (0.20,1.00) (0.25,1.00) (0.30,1.00) (0.35,1.00) (0.40,1.00) (0.45,1.00) (0.50,1.00) (0.55,1.00) (0.60,1.00) (0.65,1.00) (0.70,1.00) (0.75,1.00) (0.80,1.00) (0.85,1.00) (0.90,1.00) (0.95,1.00) (1.00,1.00) };

\draw[dashed,blue,very thick] (-1,0) -- (0,0) -- (0,1) -- (1,1);

\draw[dashed,red,thick] (-1,0) -- (-0.05,0) --  (0.05,1) -- (1,1);

\end{tikzpicture}
\hspace{1cm}
\begin{tikzpicture}[scale=2.7]

 \draw[->,gray] (-1.2,0) -- (1.2,0); 
 \draw[->,gray] (0,0) -- (0,1.2);

 \draw (-1,0.05) -- (-1,-0.05) node[below] {\scriptsize -1};
 \draw (1,0.05) -- (1,-0.05) node[below] {\scriptsize 1};
 \draw (0.05,1) -- (-0.05,1) node[left] {\scriptsize 1};

\draw[-, thick,green] plot[smooth] coordinates{(-1.00,0.00) (-0.95,0.00) (-0.90,0.00) (-0.85,0.00) (-0.80,0.00) (-0.75,0.00) (-0.70,0.00) (-0.65,0.00) (-0.60,0.00) (-0.55,0.00) (-0.50,0.00) (-0.45,0.00) (-0.40,0.00) (-0.35,0.00) (-0.30,0.01) (-0.25,0.03) (-0.20,0.07) (-0.15,0.13) (-0.10,0.23) (-0.05,0.35) (0.00,0.50) (0.05,0.65) (0.10,0.77) (0.15,0.87) (0.20,0.93) (0.25,0.97) (0.30,0.99) (0.35,1.00) (0.40,1.00) (0.45,1.00) (0.50,1.00) (0.55,1.00) (0.60,1.00) (0.65,1.00) (0.70,1.00) (0.75,1.00) (0.80,1.00) (0.85,1.00) (0.90,1.00) (0.95,1.00) (1.00,1.00) };

\draw[dotted,very thick,black] plot[smooth] coordinates{(-1.00,0.00) (-0.95,0.00) (-0.90,0.00) (-0.85,0.00) (-0.80,0.00) (-0.75,0.00) (-0.70,0.00) (-0.65,0.00) (-0.60,0.00) (-0.55,0.00) (-0.50,0.00) (-0.45,0.00) (-0.40,0.01) (-0.35,0.01) (-0.30,0.03) (-0.25,0.05) (-0.20,0.08) (-0.15,0.14) (-0.10,0.23) (-0.05,0.35) (0.00,0.50) (0.05,0.65) (0.10,0.77) (0.15,0.86) (0.20,0.92) (0.25,0.95) (0.30,0.97) (0.35,0.99) (0.40,0.99) (0.45,1.00) (0.50,1.00) (0.55,1.00) (0.60,1.00) (0.65,1.00) (0.70,1.00) (0.75,1.00) (0.80,1.00) (0.85,1.00) (0.90,1.00) (0.95,1.00) (1.00,1.00) };

\draw[dashed,blue,very thick] (-1,0) -- (0,0) -- (0,1) -- (1,1);

\draw[dashed,red,thick] (-1,0) -- (-0.1667,0) --  (0.1667,1) -- (1,1);

\end{tikzpicture}
\end{center}
\caption{\footnotesize Illustrations of transfer functions for $L=10$
  (left) and $L=3$ (right): the 0-1 transfer function (dashed blue line); the sigmoid transfer function (dotted black line); the erf transfer function (green line); the piece-wise linear transfer function (dashed red line).} \label{fig:erf}
\end{figure}
Analogously to the definition of $H_{\phi_{0-1}}$, for a general
transfer function $\phi$ we define $H_\phi$ to be the set of
predictors $\x \mapsto \phi(\inner{\w,\x})$.  Since now the range of
$\phi$ is not $\{0,1\}$ but rather the entire interval $[0,1]$, we
interpret $\phi(\inner{\w,\x})$ as the probability to output the label
$1$.  The definition of $\err_\D(h)$ remains\footnote{
Note that in this case $\err_\D(h)$ can be interpreted as 
$\prob_{(\x,y) \sim \D, b \sim \phi(\inner{\w,\x})}[y \neq b]$.}  as in \eqref{eqn:def_err}. 

The advantage of using a Lipschitz transfer function can be seen via Rademacher generalization bounds \cite{BartlettMe02}.  In fact, a simple corollary of the contraction lemma implies the following:
\begin{theorem} \label{thm:Rad} 
Let $\epsilon,\delta \in (0,1)$ and let $\phi$ be an $L$-Lipschitz
transfer function. Let $m$ be an integer satisfying
\[
m ~\ge~ \left(\frac{2L + 3\sqrt{2\ln(8/\delta)}}{\epsilon} \right)^2 ~.
\]
Then, for any distribution $\D$ over $\X \times \{0,1\}$, the ERM
algorithm  $(\epsilon,\delta)$-learns the concept class $H_\phi$ using
$m$ examples. 
\end{theorem}
The above theorem tells us that the sample complexity of learning $H_\phi$ is
$\tilde{\Omega}(L^2/\epsilon^2)$. Crucially, the sample complexity does not
depend on the dimensionality of $\X$, but only on the Lipschitz constant of the
transfer function. This allows us to learn with kernels, when the
dimensionality of $\X$ can even be infinite. A related analysis compares the
error rate of a halfspace $\w$ to the number of margin mistakes $\w$ makes on
the training set - see \secref{sec:margin} for a comparison.

From the computational complexity point of view, the result given in
\thmref{thm:Rad} is problematic, since the
ERM algorithm should solve the non-convex optimization problem
\begin{equation} \label{eqn:optERM}
\argmin_{\w: \|\w\| \le 1} ~\frac{1}{m} \sum_{i=1}^m |\phi(\inner{\w,\x_i})-y_i| ~.
\end{equation}
Solving this problem in polynomial time is hard under reasonable assumptions
(see \secref{sec:hardness} in which we present a formal hardness
result). Adapting a technique due to \cite{Ben-DavidSi00} we show in
\fullpaper{\appref{sec:ERMsolve}}\conferencepaper{the full version of this
  paper \cite{SSSS10}} that it is possible to find an $\epsilon$-accurate
solution to \eqref{eqn:optERM} (where the transfer function is $\phi_\pw$) in
time $\poly\left(\exp\left(\tfrac{L^2}{\epsilon^2}
    \log(\tfrac{L}{\epsilon})\right)\right)$.  The main contribution of this
paper is the derivation and analysis of a more simple learning algorithm that
$(\epsilon,\delta)$-learns the class $H_\sig$ using time and sample complexity
of at most $\poly\left(\exp\left(L\, \log(\tfrac{L}{\epsilon}
    )\right)\right)$. That is, the runtime of our algorithm is exponentially
smaller than the runtime required to solve the ERM problem using the technique
described in \cite{Ben-DavidSi00}. Moreover, the algorithm of
\cite{Ben-DavidSi00} performs an exhaustive search over all $(L/\epsilon)^2$
subsets of the $m$ examples in the training set, and therefore its runtime is
always order of $m^{L^2/\epsilon^2}$. In contrast, our algorithm's runtime
depends on a parameter $B$, which is bounded by $\exp(L)$ only under a
worst-case assumption. Depending on the underlying distribution, $B$ can be
much smaller than the worst-case bound. In practice, we will cross-validate for
$B$, and therefore the worst-case bound will often be pessimistic.

The rest of the paper is organized as follows. In \secref{sec:main} we describe
our main results. Next, in \secref{sec:hardness} we provide a hardness result,
showing that it is not likely that there exists an algorithm that learns
$H_\sig$ or $H_\pw$ in time polynomial in $L$. We outline additional related work in
\secref{sec:related}. In particular, the relation between our approach and
margin-based analysis is described in \secref{sec:margin}, and the relation to
approaches utilizing a distributional assumption is discussed in
\secref{sec:kalai}. We wrap up with a discussion in \secref{sec:discussion}.

\section{Main Results} \label{sec:main}

In this section we present our main result. Recall that we would like
to derive an algorithm which learns the class $H_\sig$. However, the
ERM optimization problem associated with $H_\sig$ is non-convex. The
main idea behind our construction is to learn a larger hypothesis
class, denoted $H_B$, which approximately contains $H_\sig$, and for which the ERM
optimization problem becomes convex. The price we need to pay is that
from the statistical point of view, it is more difficult to learn the
class $H_B$ than the class $H_\sig$, therefore the sample complexity
increases.  




The class $H_B$ we use is a class of \emph{linear} predictors in some
other RKHS. The kernel function that implements the inner product in
the newly constructed RKHS is
\begin{equation} \label{eqn:kerneldef}
K(\x,\x') ~\eqdef~ \frac{1}{1 - \nu \inner{\x,\x'}} ~,
\end{equation}
where $\nu \in (0,1)$ is a parameter and $\inner{\x,\x'}$ is the inner product
in the original RKHS. As mentioned previously, $\inner{\x,\x'}$ is usually
implemented by some kernel function $K'(\z,\z')$, where $\z$ and $\z'$ are the
pre-images of $\x$ and $\x'$ with respect to the feature mapping induced by
$K'$. Therefore, the kernel in \eqref{eqn:kerneldef} is simply a
composition with $K'$, i.e.  $K(\z,\z')=1/(1-\nu K'(\z,\z'))$. 

To simplify the presentation we will set $\nu = 1/2$, although in practice
other choices might be more effective. It is easy to verify that $K$ is a valid
positive definite kernel function (see for example
\cite{ScholkopfSm02,CristianiniSh04}).  Therefore, there exists some mapping
$\psi : \X \to \V$, where $\V$ is an RKHS with $\inner{\psi(\x),\psi(\x')} =
K(\x,\x')$. The class $H_B$ is defined to be:
\begin{equation} \label{eqn:HBdef}
H_B ~\eqdef~ \{ \x \mapsto \inner{\v,\psi(\x)} \, : \, \v \in \V,~\|\v\|^2 \le B \} ~.
\end{equation}

The main result we prove in this section is the following:

\begin{theorem}\label{thm:mainres}
Let $\epsilon,\delta \in (0,1)$ and let $L \ge 3$. Let $B = 2L^4+\exp\left(7L 
\log\left(\tfrac{2L}{\epsilon}\right)+3\right)$ and let $m$ be a sample size 
that satisfies $m \ge \frac{8B}{\epsilon^2} \, 
\left(2+9\sqrt{\ln(8/\delta)}\right)^2 $. Then, for any distribution $\D$, 
with probability of at least $1-\delta$, any ERM predictor $\hat{h} \in H_B$ 
with respect to $H_B$ satisfies
\[
\err_\D(\hat{h}) \le \min_{h \in H_\sig} \err_\D(h_\sig) + \epsilon ~.
\]
\end{theorem}

We note that the bound on $B$ is far from being the tightest possible in terms
of constants and second-order terms. Also, the assumption of $L\geq 3$ is
rather arbitrary, and is meant to simplify the presentation of the bound.

To prove this theorem, we start with analyzing the time and sample complexity
of learning $H_B$. The sample complexity analysis follows directly from a
Rademacher generalization bound \cite{BartlettMe02}. In particular, the
following theorem tells us that the sample complexity of learning $H_B$ with
the ERM rule is order of $B/\epsilon^2$ examples.
\begin{theorem} \label{thm:HBsample}
Let $\epsilon,\delta \in (0,1)$, let $B \ge 1$, and let $m$ be a sample size that satisfies
\[
m ~\ge~ \frac{2B}{\epsilon^2} \,
\left(2+9\sqrt{\ln(8/\delta)}\right)^2 ~.
\]
Then, for any distribution $\D$, the ERM algorithm
$(\epsilon,\delta)$-learns $H_B$. 
\end{theorem}
\begin{proof}
Since $K(\x,\x) \le 2$, the Rademacher complexity of $H_B$ is bounded
by $\sqrt{2 B/m}$ (see also \cite{KakadeSrTe08}). Additionally, using
Cauchy-Schwartz inequality we have that the loss is bounded, $|\inner{\v,\psi(\x)}-y| \le
\sqrt{2 B} + 1$. The result now follows directly from
\cite{BartlettMe02,KakadeSrTe08}. 
\end{proof}

Next, we show that the ERM problem with respect to $H_B$ can be solved
in time $\poly(m)$. The ERM problem associated with $H_B$ is
\[
\min_{\v : \|\v\|^2 \le B} \frac{1}{m} \sum_{i=1}^m |
\inner{\v,\psi(\x_i)} - y_i| ~.
\]
Since the objective function is defined only via
inner products with $\psi(\x_i)$, and the constraint on $\v$ is
defined by the $\ell_2$-norm, it follows by the Representer theorem
\cite{Wahba90} that there is an optimal solution $\v^\star$ that can
be written as $\v^\star = \sum_{i=1}^m \alpha_i
\psi(\x_i)$. Therefore, instead of optimizing over $\v$, we
can optimize over the set of weights $\alpha_1,\ldots,\alpha_m$ by solving the
equivalent optimization problem
\[
\min_{\alpha_1,\ldots,\alpha_m} \frac{1}{m} \sum_{i=1}^m \left|
\sum_{j=1}^m \alpha_j K(\x_j,\x_i) - y_i \right|  ~~~\textrm{s.t.}~~~
\sum_{i,j = 1}^m \alpha_i \alpha_j K(\x_i,\x_j) \le B~.
\]
This is a convex optimization problem in $\reals^m$ and therefore 
can be solved in time $\poly(m)$ using standard optimization tools.\footnote{ 
In fact, using stochastic gradient
descent, we can $(\epsilon,\delta)$-learn $H_B$ in time $O(m^2)$,
where $m$ is as defined in \thmref{thm:HBsample} ---See for example \cite{BottouBo08,ShalevSr08}.}
We therefore obtain:
\begin{corollary}
Let $\epsilon,\delta \in (0,1)$ and let $B \ge 1$. Then, for any
distribution $\D$, it is possible to $(\epsilon,\delta)$-learn $H_B$
in sample and time complexity of $\poly\left(\tfrac{B}{\epsilon} \,\log(1/\delta)\right)$. 
\end{corollary}

It is left to understand  why the class $H_B$ approximately contains
the class $H_\sig$. 
Recall that for any transfer function, $\phi$, we define the class
$H_\phi$ to be all the predictors of the form $\x \mapsto
\phi(\inner{\w,\x})$. The first step is to show that $H_B$ contains
the union of $H_\phi$ over all polynomial transfer functions that satisfy a certain boundedness
condition on their coefficients. 
\begin{lemma} \label{lem:PB}
Let $P_B$ be the following set
of polynomials (possibly with infinite degree)
\begin{equation} \label{eqn:PB}
P_{B} \eqdef \left\{ p(a) =  \sum_{j=0}^\infty \beta_j \,a^j \,:\,
 \sum_{j=0}^\infty \beta_j^2 \, 2^{j} \le B \right\} ~.
\end{equation}
Then,
\[
\bigcup_{p \in P_B} H_p ~\subset~ H_B ~.
\]
\end{lemma}
\begin{proof}
To simplify the proof, we first assume that $\X$ is simply the unit
ball in $\reals^n$, for an arbitrarily large but finite $n$. Consider the mapping $\psi:\X \rightarrow \reals^{\mathbb{N}}$ defined as follows: for any $\x \in \X$, we let $\psi(\x)$ be an infinite vector, indexed by $k_{1}\ldots,k_{j}$ for all $(k_{1},\ldots,k_{j})\in \{1,\ldots,n\}^j$ and $j=0\ldots\infty$, where the entry at index $k_{1}\ldots,k_{j}$ equals
$
2^{-j/2} x_{k_{1}}\cdot x_{k_{2}}\cdots x_{k_{j}} $.
The inner-product between $\psi(\x)$ and $\psi(\x')$ for any $\x,\x'
\in \X$ can be calculated as follows, 
\begin{align*}
\inner{\psi(\x),\psi(\x')}& ~=~ \sum_{j=0}^{\infty}\sum_{(k_{1},\ldots,k_{j})\in \{1,\ldots,n\}^j}2^{-j}x_{k_1}x'_{k_1}\cdots x_{k_j}x'_{k_j}
~=~ \sum_{j=0}^{\infty}2^{-j}(\inner{\x,\x'})^j 
~=~ \frac{1}{1- \thalf \inner{\x,\x'}}.
\end{align*}
This is exactly the kernel function defined in \eqref{eqn:kerneldef}
(recall that we set $\nu = 1/2$) and therefore $\psi$ maps to the RKHS defined by $K$. 
Consider any polynomial $p(a)=\sum_{j=0}^{\infty} \beta_j a^j$ in
$P_{B}$, and any $\w \in \X$. Let $\v_{\w}$ be an element in
$\reals^{\mathbb{N}}$ explicitly defined as being equal to $\beta_j
2^{j/2} w_{k_1}\cdots w_{k_j}$ at index $k_1,\ldots,k_j$ (for all
$k_1,\ldots,k_j \in \{1,\ldots,n\}^j,j=0\ldots \infty$). 
By definition of $\psi$ and $\v_{\w}$, we have that
\begin{align*}
\inner{\v_{\w},\psi(\x)} & = \sum_{j=0}^{\infty}\sum_{k_1,\ldots,k_j}2^{-j/2}\beta_j 2^{j/2}w_{k_1}\cdots w_{k_j}x_{k_{1}}\cdot \cdots x_{k_{j}}
 = \sum_{j=0}^{\infty}\beta_j(\inner{\w,\x})^j 
= p(\inner{\w,\x})~.
\end{align*}
In addition,
\begin{align*}
\|\v_{\w}\|^2 & =  \sum_{j=0}^\infty \sum_{k_1,\ldots,k_j}\beta_j^2 2^{j}  w_{k_1}^2\cdots w_{k_j}^2
= \sum_{j=0}^\infty \beta_j^2 2^{j} \sum_{k_1}w_{k_1}^2\sum_{k_2}w_{k_2}^2\cdots \sum_{k_j}w_{k_j}^2
 = \sum_{j=0}^\infty \beta_j^2 2^{j} \left(\|\w\|^2\right)^j \leq B.
\end{align*}
Thus, the predictor $\x \mapsto \inner{\v_\w,\psi(\x)}$ belongs to $H_B$ and is
the same as the predictor $\x \mapsto p(\inner{\w,\x})$. This proves that $H_p
\subset H_B$ for all $p \in P_B$ as required. Finally, if $\X$ is an infinite
dimensional RKHS, the only technicality is that in order to represent $\x$ as a
(possibly infinite) vector, we need to show that our RKHS has a countable
basis. This holds since the inner product $\inner{\x,\x'}$ over $\X$ is
continuous and bounded (see \citep{Berlinet03}).
\end{proof}

Finally, the following lemma states that with a sufficiently large
$B$, there exists a polynomial in $P_B$ which approximately equals to
$\phi_\sig$. This implies that $H_B$ approximately contains $H_\sig$.  
\begin{lemma}\label{lem:sig}
Let $\phi_{\sig}$ be as defined in \eqref{eqn:sigdef}, where for simplicity we 
assume $L \geq 3$. For any $\epsilon>0$, let
\[
B = 2L^4+\exp\left(7L \log\left(\tfrac{2L}{\epsilon}\right)+3\right).
\]
Then there exists $p\in P_{B}$ such that
\[
\forall \x,\w \in \X,~~|p(\inner{\w,\x}) - \phi_\sig(\inner{\w,\x})|
\le \epsilon ~.
\]
\end{lemma}
The proof of the lemma is based on a Chebyshev approximation technique and is
given in \fullpaper{\appref{sec:Sigmoid}}\conferencepaper{the full version of
  our paper \cite{SSSS10}}.  Since the proof is rather involved, we also
present a similar lemma, whose proof is simpler, for the $\phi_\erf$ transfer
function (see \fullpaper{\appref{sec:erf}}\conferencepaper{\cite{SSSS10}}). It
is interesting to note that $\phi_{\erf}$ actually \emph{belongs} to $P_B$ for
a sufficiently large $B$, since it can be defined via its infinite-degree
Taylor expansion. However, the bound for $\phi_\erf$ depends on $\exp(L^2)$,
rather than $\exp(L)$ for the sigmoid transfer function $\phi_{\sig}$.

Finally, \thmref{thm:mainres} is obtained as follows: Combining
\thmref{thm:HBsample} and \lemref{lem:PB} we get that with probability of at
least $1-\delta$,
\begin{equation} \label{eqn:combeq}
\err_\D(\hat{h}) ~\le~ \min_{h \in H_B} \err_\D(h)  + \epsilon/2 
\le \min_{p \in P_B} \min_{h \in H_p} \err_\D(h) + \epsilon/2~.
\end{equation}
From \lemref{lem:sig} we obtain that for any $\w \in \X$, if $h(\x) =
\phi_\sig(\inner{\w,\x})$ then there exists a polynomial $p_0 \in P_B$
such that if $h'(\x) = p_0(\inner{\w,\x})$ then $ \err_\D(h') \le \err_\D(h) + \epsilon/2 $.  Since it
holds for all $\w$, we get that
\[
\min_{p \in P_B} \min_{h \in H_p} \err_\D(h) \le \min_{h \in H_\sig}
\err_\D(h) + \epsilon/2 ~.
\]
Combining this with \eqref{eqn:combeq}, \thmref{thm:mainres} follows.

\section{Hardness} \label{sec:hardness}

In this section we derive a hardness result for agnostic learning of $H_\sig$ 
or $H_\pw$ with respect to the zero-one loss. The hardness result relies on 
the hardness of standard (non-agnostic)\footnote{In the \emph{standard} PAC 
model,
  we assume that some hypothesis in the class has $\err_\D(h)=0$,
  while in the \emph{agnostic} PAC model, which we study in this
  paper, $\err_\D(h)$ might be strictly greater than zero for all $h
  \in H$. Note that our definition of $(\epsilon,\delta)$-learning in
  this paper is in the agnostic model.} PAC learning of intersection of
halfspaces given in Klivans and Sherstov~\cite{KlivansSh06} (see also
similar arguments in \cite{FeldmanGoKhPo06}). The
hardness result is representation-independent ---it makes no
restrictions on the learning algorithm and in particular also holds 
for improper learning algorithms. The hardness result is based on the
following cryptographic assumption:
\begin{assumption} \label{crypto}
There is no polynomial time solution to the
$\tilde{O}(n^{1.5})$-unique-Shortest-Vector-Problem.
\end{assumption}
In a nutshell, given a basis $\v_1,\ldots,\v_n\in \reals^n$, the
$\tilde{O}(n^{1.5})$-unique-Shortest-Vector-Problem consists of finding the
shortest nonzero vector in $\{a_1\v_1+\ldots+a_n\v_n:a_1,\ldots,a_n\in
\mathcal{Z}\}$, even given the information that it is shorter by a factor of at
least $\tilde{O}(n^{1.5})$ than any other non-parallel vector. This problem is
believed to be hard - there are no known sub-exponential algorithms, and it is
known to be NP-hard if $\tilde{O}(n^{1.5})$ is replaced by a small constant
(see \cite{KlivansSh06} for more details).

With this assumption, Klivans and Sherstov proved the following: 
\begin{theorem}[Theorem 1.2 in Klivans and
  Sherstov~\cite{KlivansSh06}] \label{thm:Klivans} Let $\X = \{\pm
  1\}^n$, let \[H = \{\x \mapsto \phi_{0,1}(\inner{\w,\x} - \theta -
  1/2) : \theta \in \N, \w \in \N^n, |\theta|+\|\w\|_1 \le poly(n)\}
  ~,\] and let $H_k = \{\x \mapsto (h_1(\x) \land \ldots \land
  h_k(\x)) : \forall i, h_i \in H\}$.  Then, based on Assumption
  \ref{crypto}, $H_k$ is not efficiently learnable in the standard PAC
  model for any $k = n^{\rho}$ where $\rho > 0$ is a constant.
\end{theorem}

The above theorem implies the following. 
\begin{lemma} \label{lem:hard1} Based on Assumption \ref{crypto},
  there is no algorithm that runs in time $\poly(n,1/\epsilon,1/\delta)$
  and $(\epsilon,\delta)$-learns the class $H$ defined in
  \thmref{thm:Klivans}.
\end{lemma}
\begin{proof}
  To prove the lemma we show that if there is a polynomial time
  algorithm that learns $H$ in the \emph{agnostic} model, then there
  exists a weak learning algorithm (with a polynomial edge) that
  learns $H_k$ in the standard (non-agnostic) PAC model.  In the standard PAC
  model, weak learning implies strong learning \cite{Schapire90}, hence
  the existence of a weak learning algorithm that learns $H_k$
  will contradict \thmref{thm:Klivans}.

  Indeed, let $\D$ be any distribution such that there exists $h^\star
  \in H_k$ with $\err_\D(h^\star) = 0$. Let us rewrite $h^\star = h_1^\star
  \land \ldots \land h_k^\star$ where for all $i$, $h^\star_i \in H$. 
To show that there exists a weak learner, we first show that there
exists some $h \in H$ with $\err_\D(h) \le 1/2 - 1/2k^2$. 

Since for each $\x$ if $h^\star(\x)=0$ then there exists $j$
s.t. $h^\star_j(\x)=0$, we can use the union bound to get that
\[
1 = \prob[\exists j : h^\star_j(\x) = 0 | h^\star(\x)=0] \le \sum_j
\prob[h^\star_j(\x)=0  | h^\star(\x)=0] 
\le k \max_j \prob[h^\star_j(\x)=0  | h^\star(\x)=0] ~.
\]
So, for $j$ that maximizes $\prob[h^\star_j(\x)=0 | h^\star(\x)=0]$ we get that 
$\prob[h^\star_j(\x)=0  | h^\star(\x)=0] \ge 1/k$. 
Therefore, 
\begin{align*}
\err_\D(h^\star_j) &= \prob[h^\star_j(\x)=1 \land h^\star(\x)=0] = 
\prob[h^\star(\x)=0] \, \prob[h^\star_j(\x) = 1 | h^\star(\x)=0] \\
&= 
\prob[h^\star(\x)=0] \, (1 - \prob[h^\star_j(\x) = 0 | h^\star(\x)=0])
\le \prob[h^\star(\x)=0] \, (1 - 1/k) ~.
\end{align*}
Now, if $\prob[h^\star(\x)=0] \le 1/2+1/k^2$ then the above gives
\[\err_\D(h^\star_j)
\le (1/2 + 1/k^2)(1-1/k) \leq 1/2-1/2k^2 ~,
\]
where the inequality holds for any positive integer $k$. 
Otherwise, if $\prob[h^\star(\x)=0] > 1/2+1/k^2$, then the constant
predictor $h(\x) = 0$ has $\err_\D(h) < 1/2 - 1/k^2$. In both cases we
have shown that there exists a predictor in $H$ with error of at most $1/2 -
1/2k^2$.

Finally, if we can agnostically learn $H$ in time
$\poly(n,1/\epsilon,1/\delta)$, then we can find $h'$ with $\err_\D(h') \le
\min_{h \in H} \err_\D(h) + \epsilon \le 1/2 - 1/2k^2 + \epsilon$ in time
$\poly(n,1/\epsilon,1/\delta)$ (recall that $k=n^\rho$ for some $\rho>0$). This
means that we can have a weak learner that runs in polynomial time, and this
concludes our proof.
\end{proof}

Let $h$ be a hypothesis in the class $H$ defined in
\thmref{thm:Klivans} and take any $\x \in \{\pm 1\}^n$. Then, there
exist an integer $\theta$ and a vector of integers $\w$ such that
$h(\x) = \phi_{0,1}(\inner{\w,\x}-\theta-1/2)$. But since
$\inner{\w,\x}-\theta$ is also an integer, if we let $L = 1$ this means that $h(\x) =
\phi_{\pw}(\inner{\w,\x}-\theta-1/2)$ as well. Furthermore, letting
$\x' \in \reals^{n+1}$ denote the concatenation of $\x$ with the
constant $1$ and letting $\w' \in \reals^{n+1}$ denote the
concatenation of $\w$ with the scalar $(-\theta-1/2)$ we obtain that 
$h(\x) = \phi_{\pw}(\inner{\w',\x'})$. Last, let us normalize $\tilde{\w}
= \w'/\|\w'\|$, $\tilde{\x}=\x/\|\x'\|$, and redefine $L$ to be
$\|\w'\|\,\|\x'\|$, we get that $h(\x) =
\phi_{\pw}(\inner{\tilde{\w},\tilde{\x}})$. That is, we have shown
that $H$ is contained in a class of the form $H_\pw$ with a Lipschitz
constant bounded by $\poly(n)$. Combining the above with
\lemref{lem:hard1} we obtain the following: 
\begin{corollary} \label{cor:hardpw}
Let $L$ be a Lipschitz constant and let $H_\pw$ be the class defined
by the $L$-Lipschitz transfer function $\phi_\pw$. Then, based on
Assumption \ref{crypto}, 
there is no algorithm that runs in time $\poly(L,1/\epsilon,1/\delta)$
and $(\epsilon,\delta)$-learns the class $H_\pw$. 
\end{corollary}

A similar argument leads to the hardness of learning $H_\sig$.
\begin{theorem}
Let $L$ be a Lipschitz constant and let $H_\sig$ be the class defined
by the $L$-Lipschitz transfer function $\phi_\sig$. Then, based on
Assumption \ref{crypto}, 
there is no algorithm that runs in time $\poly(L,1/\epsilon,1/\delta)$
and $(\epsilon,\delta)$-learns the class $H_\sig$. 
\end{theorem}
\begin{proof}
Let $h$ be a hypothesis in the class $H$ defined in
\thmref{thm:Klivans} and take any $\x \in \{\pm 1\}^n$. Then, there exist an integer $\theta$ and a vector of integers $\w$ such that $h(\x) = \phi_{0,1}(\inner{\w,\x}-\theta-1/2)$. However, since $\inner{\w,\x}-\theta$ is also an integer, we see that 
$$
|\phi_{0,1}(\inner{\w,\x}-\theta-1/2) -
\phi_{\sig}(\inner{\w,\x}-\theta-1/2)| \le \frac{1}{1 + \exp(2 L )} ~.
$$
This means that for any $\epsilon > 0$, if we pick 
$
L = \frac{\log(2/\epsilon - 1)}{2}
$ 
and define $h_\sig(\x) = \phi_{\sig}(\inner{\w,\x}-\theta-1/2)$, then $|h(\x) - h_\sig(\x)| \le \epsilon/2$.
 Furthermore, letting $\x' \in \reals^{n+1}$ denote the concatenation of $\x$ with the constant $1$ and letting $\w' \in \reals^{n+1}$ denote the concatenation of $\w$ with the scalar $(-\theta-1/2)$ we obtain that $h_\sig(\x) = \phi_{\sig}(\inner{\w',\x'})$. Last, let us normalize $\tilde{\w}
= \w'/\|\w'\|$, $\tilde{\x}=\x/\|\x'\|$, and redefine $L$ to be
\begin{align}\label{eq:ell}
L = \frac{\|\w'\| \|\x'\| \log(2/\epsilon - 1)}{2}
\end{align}
so that $h_\sig(\x) =
\phi_{\sig}(\inner{\tilde{\w},\tilde{\x}})$. Thus we see that if there
exists an algorithm that runs in time $\poly(L,1/\epsilon,1/\delta)$
and $(\epsilon/2,\delta)$-learns the class $H_\sig$, then since for
all $h \in H$ exists $h_\sig \in H_\sig$ such that $
|h_\sig(\x) - h(\x)| \le \epsilon/2 $, there also exists an algorithm
that $(\epsilon,\delta)$-learns the concept class $H$ defined in
\thmref{thm:Klivans} in time polynomial in $(L,1/\epsilon,1/\delta)$
(for $L$ defined in Equation \ref{eq:ell}). But by definition of $L$
in Equation \ref{eq:ell} and the fact that $\|\w'\|$ and $\|\x'\|$ are
of size $\poly(n)$, this means that there is an algorithm that runs in
time polynomial in $(n,1/\epsilon,1/\delta)$ and
$(\epsilon,\delta)$-learns the class $H$, which contradicts
\lemref{lem:hard1}.
\end{proof}

\section{Related work} \label{sec:related}

The problem of learning kernel-based halfspaces has been extensively
studied before, mainly in the framework of SVM
\cite{Vapnik98,CristianiniSh04,ScholkopfSm02}.  When the data is
separable with a margin $\mu$, it is possible to learn a halfspaces in
polynomial time. The learning problem becomes much more difficult when
the data is not separable with margin.

In terms of hardness results, \cite{Ben-DavidSi00} derive hardness
results for proper learning with sufficiently small margins. There are
also strong hardness of approximation results for \emph{proper}
learning \emph{without} margin (see for example
\citep{GuruswamiRa06} and the references therein). We
emphasize that we allow improper learning, which is just as useful for
the purpose of learning good classifiers, and thus these hardness
results do not apply. Instead, the hardness result we derived in
\secref{sec:hardness} hold for improper learning as well. As mentioned
before, the main tool we rely on for deriving the hardness result is
the representation independent hardness result for learning
intersections of halfspaces given in \cite{KlivansSh06}.

Practical algorithms such as SVM often replace the 0-1 error function with a convex
surrogate, and then apply convex optimization tools. However, there
are no guarantees on how well the surrogate function approximates the
0-1 error function. Recently, \cite{Zhang04a,BartlettJoMc06} studied
the \emph{asymptotic} relationship between surrogate convex loss
functions and the 0-1 error function. In contrast, in this paper we
show that even with a finite sample, surrogate convex loss functions
can be competitive with the 0-1 error function as long as we replace
inner-products with the kernel $K(\x,\x') = 1/(1-0.5\inner{\x,\x'})$.

\subsection{Margin analysis} \label{sec:margin}

Recall that we circumvented the dependence of the VC dimension of
$H_{\phi_{0-1}}$ on the dimensionality of $\X$ by replacing $\phi_{0-1}$ with a
Lipschitz transfer function. Another common approach is to require that the
learned classifier will be competitive with the \emph{margin} error rate of the
optimal halfspace. Formally, the $\mu$-margin error rate of a halfspace of the
form $h_\w(\x) = \indct{\inner{\w,\x} > 0}$ is defined as:
\begin{equation} \label{eqn:errmu}
\err_{\D,\mu}(\w) ~=~ \Pr[h_\w(\x) \neq y \lor |\inner{\w,\x}| \leq \mu] ~.
\end{equation}
Intuitively, $\err_{\D,\mu}(\w)$ is the error rate of $h_\w$ had we
$\mu$-shifted each point in the worst possible way. Margin based
analysis restates the goal of the learner (as given in \eqref{eqn:aPAC})
and requires that the learner will find a classifier $h$ that
satisfies:
\begin{equation} \label{eqn:mb}
\err_\D(h) ~\le~ \min_{\w : \|\w\| = 1} \err_{\D,\mu}(\w) + \epsilon ~.
\end{equation}
Bounds of the above form are called margin-based bounds and are widely
used in the statistical analysis of Support Vector Machines and
AdaBoost. It was shown \citep{BartlettMe02,McAllester03} that $m =
\Theta(\log(1/\delta)/(\mu\,\epsilon)^2)$ examples are sufficient (and
necessary) to learn a classifier for which \eqref{eqn:mb} holds with
probability of at least $1-\delta$. Note that as in the sample
complexity bound we gave in \thmref{thm:Rad}, the margin based sample
complexity bound also does not depend on the dimension. 

In fact, the Lipschitz approach used in this paper and the margin-based
approach are closely related. First, it is easy to verify that if we set $L =
1/(2\mu)$, then for any $\w$ the hypothesis $h(\x) =\phi_{\pw}(\inner{\w,\x})$
satisfies $\err_\D(h) \le \err_{\D,\mu}(\w)$. Therefore, an algorithm that
$(\epsilon,\delta)$-learns $H_\pw$ also guarantees that \eqref{eqn:mb}
holds. Second, it is also easy to verify that if we set $L =
\tfrac{1}{4\mu}\log\left(\tfrac{2-\epsilon}{\epsilon}\right)$ then for any $\w$
the hypothesis $h(\x) =\phi_{\sig}(\inner{\w,\x})$ satisfies $\err_\D(h) \le
\err_{\D,\mu}(\w) + \epsilon/2$. Therefore, an algorithm that
$(\epsilon/2,\delta)$-learns $H_\sig$ also guarantees that \eqref{eqn:mb}
holds.

As a direct corollary of the above discussion we obtain that it is
possible to learn a vector $\w$ that guarantees \eqref{eqn:mb} in time
$\poly(\exp(\tilde{O}(1/\mu)))$. 

A computational complexity analysis under margin assumptions was first carried
out in \citep{Ben-DavidSi00} (see also the hierarchical worst-case analysis
recently proposed in \citep{Ben-David06}).  The technique used in
\citep{Ben-DavidSi00} is based on the observation that in the noise-free case,
an optimal halfspace can be expressed as a linear sum of at most $1/\mu^2$
examples. Therefore, one can perform an exhaustive search over all
sub-sequences of $1/\mu^2$ examples, and choose the optimal halfspace. Note
that this algorithm will always run in time $m^{1/\mu^2}$. Since the sample
complexity bound requires that $m$ will be order of $1/(\mu\epsilon)^2$, the
runtime of the method described by \citep{Ben-DavidSi00} becomes
$\poly(\exp(\tilde{O}(1/\mu^2)))$. In comparison, our algorithm achieves a
better runtime of $\poly(\exp(\tilde{O}(1/\mu)))$. Moreover, while the
algorithm of \cite{Ben-DavidSi00} performs an exhaustive search, our
algorithm's runtime depends on the parameter $B$, which is
$\poly(\exp(\tilde{O}(1/\mu)))$ only under a worst-case assumption. Since in
practice we will cross-validate for $B$, it is plausible that in many
real-world scenarios the runtime of our algorithm will be much smaller.

\subsection{Distributional Assumptions} \label{sec:kalai}

The idea of approximating the zero-one transfer function with a
polynomial was first proposed by \citep{KalaiKlMaSe05} who studied the
problem of agnostically learning halfspaces without kernels in
$\reals^n$ under distributional assumption. In particular, they showed
that if the distribution over $\X$ is uniform over the unit ball, then
it is possible to agnostically learn $H_{\phi_{0-1}}$ in time
$\poly(n^{1/\epsilon^4})$.  This was further generalized by
\citep{BlaisOdWi08}, who showed that similar bounds hold for product
distributions.

Beside distributional assumptions, these works are characterized by
explicit dependence on the dimension of $\X$, and therefore are not
adequate for the kernel-based setting we consider in this paper, in
which the dimensionality of $\X$ can even be infinite. More precisely,
while \citep{KalaiKlMaSe05} try to approximate the zero-one transfer
function with a low-degree polynomial, we require instead that the
coefficients of the polynomials are bounded. The principle that when
learning in high dimensions ``the size of the parameters is more
important than their number'' was one of the main advantages in the
analysis of the statistical properties of several learning algorithms
(e.g. \citep{Bartlett96}).

Interestingly, in \cite{ShalevShSr09tech} we show that the very same
algorithm we use in this paper recover the same complexity bound of
\cite{KalaiKlMaSe05}.

\section{Discussion} \label{sec:discussion}

In this paper we described and analyzed a new technique for
agnostically learning kernel-based halfspaces with the zero-one loss
function.  The bound we derive has an exponential dependence on $L$,
the Lipschitz coefficient of the transfer function. While we prove
that (under a certain cryptographic assumption) no algorithm can have
a polynomial dependence on $L$, the immediate open question is whether
the dependence on $L$ can be further improved. 

A perhaps surprising property of our analysis is that we propose a
single algorithm, returning a single classifier, which is
simultaneously competitive against \emph{all} transfer functions $p\in
P_{B}$. In particular, it learns with respect to the ``optimal''
transfer function, where by optimal we mean the one which attains the
smallest error rate, $\E[|p(\inner{\w,\x})-y|]$, over the distribution
$\D$.

Our algorithm boils down to linear regression with the absolute loss function
and while composing a particular kernel function over our original RKHS. It is
possible to show that solving the vanilla SVM, with the hinge-loss, and
composing again our particular kernel over the desired kernel, can also give
similar guarantees. It is therefore interesting to study if there is something
special about the kernel we propose or maybe other kernel functions (e.g. the
Gaussian kernel) can give similar guarantees.

Another possible direction is to consider other types of margin-based
analysis or transfer functions. For example, in the statistical
learning literature, there are several definitions of ``noise''
conditions, some of them are related to margin, which lead to faster
decrease of the error rate as a function of the number of examples
(see for example \cite{Bousquet02,Tsybakov04,Steinwart07}). Studying
the computational complexity of learning under these conditions is
left to future work.

\section*{Acknowledgments} 
We would like to thank Adam Klivans for helping with the Hardness
results. This work was partially supported by a Google Faculty
Research Grant.

\bibliographystyle{plainnat}
\bibliography{bib}

\begin{thebibliography}{28}
\providecommand{\natexlab}[1]{#1}
\providecommand{\url}[1]{\texttt{#1}}
\expandafter\ifx\csname urlstyle\endcsname\relax
  \providecommand{\doi}[1]{doi: #1}\else
  \providecommand{\doi}{doi: \begingroup \urlstyle{rm}\Url}\fi

\bibitem[A.~Berlinet(2003)]{Berlinet03}
C.~Thomas-Agnan A.~Berlinet.
\newblock \emph{Reproducing Kernel Hilbert Spaces in Probability and
  Statistics}.
\newblock Springer, 2003.

\bibitem[Bartlett(1997)]{Bartlett96}
P.~L. Bartlett.
\newblock For valid generalization, the size of the weights is more important
  than the size of the network.
\newblock In \emph{Advances in Neural Information Processing Systems 9}, 1997.

\bibitem[Bartlett and Mendelson(2002)]{BartlettMe02}
P.~L. Bartlett and S.~Mendelson.
\newblock Rademacher and {G}aussian complexities: {R}isk bounds and structural
  results.
\newblock \emph{Journal of Machine Learning Research}, 3:\penalty0 463--482,
  2002.

\bibitem[Bartlett et~al.(2006)Bartlett, Jordan, and McAuliffe]{BartlettJoMc06}
P.~L. Bartlett, M.~I. Jordan, and J.~D. McAuliffe.
\newblock Convexity, classification, and risk bounds.
\newblock \emph{Journal of the American Statistical Association}, 101:\penalty0
  138--156, 2006.

\bibitem[Ben-David(2006)]{Ben-David06}
S.~Ben-David.
\newblock Alternative measures of computational complexity.
\newblock In \emph{TAMC}, 2006.

\bibitem[Ben-David and Simon(2000)]{Ben-DavidSi00}
S.~Ben-David and H.~Simon.
\newblock Efficient learning of linear perceptrons.
\newblock In \emph{NIPS}, 2000.

\bibitem[Blais et~al.(2008)Blais, O'Donnell, and Wimmer]{BlaisOdWi08}
E.~Blais, R.~O'Donnell, and K~Wimmer.
\newblock Polynomial regression under arbitrary product distributions.
\newblock In \emph{COLT}, 2008.

\bibitem[Bottou and Bousquet(2008)]{BottouBo08}
L.~Bottou and O.~Bousquet.
\newblock The tradeoffs of large scale learning.
\newblock In \emph{NIPS}, pages 161--168, 2008.

\bibitem[Bousquet(2002)]{Bousquet02}
O.~Bousquet.
\newblock \emph{Concentration Inequalities and Empirical Processes Theory
  Applied to the Analysis of Learning Algorithms}.
\newblock PhD thesis, Ecole Polytechnique, 2002.

\bibitem[Cristianini and Shawe-Taylor(2004)]{CristianiniSh04}
N.~Cristianini and J.~Shawe-Taylor.
\newblock \emph{Kernel Methods for Pattern Analysis}.
\newblock Cambridge University Press, 2004.

\bibitem[Elliot(1964)]{Elliot64}
D.~Elliot.
\newblock The evaluation and estimation of the coefficients in the chebyshev
  series expansion of a function.
\newblock \emph{Mathematics of Computation}, 18\penalty0 (86):\penalty0
  274--284, April 1964.

\bibitem[Feldman et~al.(2006)Feldman, Gopalan, Khot, and
  Ponnuswami]{FeldmanGoKhPo06}
V.~Feldman, P.~Gopalan, S.~Khot, and A.K. Ponnuswami.
\newblock New results for learning noisy parities and halfspaces.
\newblock In \emph{In Proceedings of the 47th Annual IEEE Symposium on
  Foundations of Computer Science}, 2006.

\bibitem[Guruswami and Raghavendra(2006)]{GuruswamiRa06}
V.~Guruswami and P.~Raghavendra.
\newblock Hardness of learning halfspaces with noise.
\newblock In \emph{Proceedings of the 47th Foundations of Computer Science
  (FOCS)}, 2006.

\bibitem[Kakade et~al.(2008)Kakade, Sridharan, and Tewari]{KakadeSrTe08}
S.M. Kakade, K.~Sridharan, and A.~Tewari.
\newblock On the complexity of linear prediction: Risk bounds, margin bounds,
  and regularization.
\newblock In \emph{NIPS}, 2008.

\bibitem[Kalai et~al.(2005)Kalai, Klivans, Mansour, and
  Servedio]{KalaiKlMaSe05}
A.~Kalai, A.R. Klivans, Y.~Mansour, and R.~Servedio.
\newblock Agnostically learning halfspaces.
\newblock In \emph{Proceedings of the 46th Foundations of Computer Science
  (FOCS)}, 2005.

\bibitem[Kearns et~al.(1992)Kearns, Schapire, and Sellie]{KearnsScSe92}
M.~J. Kearns, R.~E. Schapire, and L.~M. Sellie.
\newblock Toward efficient agnostic learning.
\newblock In \emph{COLT}, pages 341--352, July 1992.
\newblock To appear, {\it Machine Learning}.

\bibitem[Klivans and Sherstov(2006)]{KlivansSh06}
Adam~R. Klivans and Alexander~A. Sherstov.
\newblock Cryptographic hardness for learning intersections of halfspaces.
\newblock In \emph{FOCS}, 2006.

\bibitem[Mason(2003)]{Mason03}
J.C. Mason.
\newblock \emph{Chebyshev Polynomials}.
\newblock CRC Press, 2003.

\bibitem[McAllester(2003)]{McAllester03}
D.~A. McAllester.
\newblock Simplified {PAC}-{B}ayesian margin bounds.
\newblock In \emph{COLT}, pages 203--215, 2003.

\bibitem[Schapire(1990)]{Schapire90}
R.E. Schapire.
\newblock The strength of weak learnability.
\newblock \emph{Machine Learning}, 5\penalty0 (2):\penalty0 197--227, 1990.

\bibitem[Sch{\"o}lkopf and Smola(2002)]{ScholkopfSm02}
B.~Sch{\"o}lkopf and A.~J. Smola.
\newblock \emph{Learning with Kernels: Support Vector Machines, Regularization,
  Optimization and Beyond}.
\newblock MIT Press, 2002.

\bibitem[Shalev-Shwartz and Srebro(2008)]{ShalevSr08}
S.~Shalev-Shwartz and N.~Srebro.
\newblock {SVM} optimization: Inverse dependence on training set size.
\newblock In \emph{International Conference on Machine Learning}, pages
  928--935, 2008.

\bibitem[Shalev-Shwartz et~al.(2009)Shalev-Shwartz, Shamir, and
  Sridharan]{ShalevShSr09tech}
S.~Shalev-Shwartz, O.~Shamir, and K.~Sridharan.
\newblock Agnostically learning halfspaces with margin errors.
\newblock Technical report, Toyota Technological Institute, 2009.

\bibitem[Steinwart and Scovel(2007)]{Steinwart07}
I.~Steinwart and C.~Scovel.
\newblock Fast rates for support vector machines using gaussian kernels.
\newblock \emph{Annals of Statistics}, 35:\penalty0 575, 2007.

\bibitem[Tsybakov(2004)]{Tsybakov04}
A.~Tsybakov.
\newblock Optimal aggregation of classifiers in statistical learning.
\newblock \emph{Annals of Statistics}, 32:\penalty0 135--166, 2004.

\bibitem[Vapnik(1998)]{Vapnik98}
V.~N. Vapnik.
\newblock \emph{Statistical Learning Theory}.
\newblock Wiley, 1998.

\bibitem[Wahba(1990)]{Wahba90}
G.~Wahba.
\newblock \emph{Spline Models for Observational Data}.
\newblock SIAM, 1990.

\bibitem[Zhang(2004)]{Zhang04a}
T.~Zhang.
\newblock Statistical behavior and consistency of classification methods based
  on convex risk minimization.
\newblock \emph{The Annals of Statistics}, 32:\penalty0 56--85, 2004.

\end{thebibliography}

\conferencepaper{\end{document}}

\newpage
\appendix

\section{Solving the ERM problem given in \eqref{eqn:optERM}} \label{sec:ERMsolve}

In this section we show how to approximately solve \eqref{eqn:optERM}
when the transfer function is $\phi_\pw$. The technique we use is
similar to the covering technique described in \cite{Ben-DavidSi00}. 

For each $i$, let $b_i =
2(y_i - 1/2)$. It is easy to verify that the objective of
\eqref{eqn:optERM} can be rewritten as
\begin{equation} \label{eqn:optERM2}
\frac{1}{m} \sum_{i=1}^m f(b_i \inner{\w,\x_i}) ~~~~\textrm{where}~~
f(a) = \min\{1,\max\{0,1/2-L\,a\}\} ~.
\end{equation}
Let $g(a) =\max\{0,1/2-L\,a\}$. Note that $g$ is a convex function,
$g(a) \ge f(a)$ for every $a$, and equality holds whenever $a \ge
-1/2L$. 

Let $\w^\star$ be a minimizer 
of \eqref{eqn:optERM2} over the unit ball. We partition the set $[m]$ into 
\[
I_1 = \{i \in [m] : g(b_i\inner{\w^\star,\x_i}) = f(b_i \inner{\w^\star,\x_i})\}  ~~~,~~~ I_2 =
[m] \setminus I_1 ~.
\]
Now, let $\hat{\w}$ be a vector that satisfies 
\begin{equation} \label{eqn:regsuff}
\sum_{i \in I_1} g(b_i\inner{\hat{\w},\x_i}) ~\le~ 
\min_{\w : \|\w\| \le 1} \sum_{i \in I_1} g(b_i
\inner{\w,\x_i}) + \epsilon\,m ~.
\end{equation}
Clearly, we have
\begin{equation*}
\begin{split}
\sum_{i=1}^m f(b_i\inner{\hat{\w},\x_i}) &\le
\sum_{i \in I_1} g(b_i\inner{\hat{\w},\x_i}) + 
\sum_{i \in I_2} f(b_i\inner{\hat{\w},\x_i}) \\ 
&\le
\sum_{i \in I_1} g(b_i\inner{\hat{\w},\x_i}) +
|I_2| \\
&\le
\sum_{i \in I_1} g(b_i\inner{\w^\star,\x_i}) + \epsilon\,m + 
|I_2| \\
&=\sum_{i=1}^m f(b_i\inner{\w^\star,\x_i})  + \epsilon\,m ~.
\end{split}
\end{equation*}
Dividing the two sides of the above by $m$ we obtain that $\hat{\w}$ is an
$\epsilon$-accurate solution to \eqref{eqn:optERM2}. 
Therefore, it suffices to show a method that finds a vector $\hat{\w}$
that satisfies \eqref{eqn:regsuff}. To do so, we use a standard
generalization bound (based on Rademacher complexity) as follows: 
\begin{lemma}
Let us sample $i_1,\ldots,i_k$ i.i.d. according to
the uniform distribution over $I_1$. Let $\hat{\w}$ be a minimizer of 
$\sum_{j=1}^k g(b_{i_j} \inner{\w,\x_{i_j}})$ over $\w$ in the
unit ball. Then, 
\[
\E\left[ \tfrac{1}{|I_1|}\sum_{i \in I_1} g(b_i\inner{\hat{\w},\x_i}) -
\min_{\w : \|\w\| \le 1} \tfrac{1}{|I_1|}\sum_{i \in I_1} g(b_i
\inner{\w,\x_i})  \right] ~\le~ 2 L / \sqrt{k} ~,
\]
where expectation is over the choice of $i_1,\ldots,i_k$. 
\end{lemma}
\begin{proof}
Simply note that $g$ is $L$-Lipschitz and then apply a Rademacher
generalization bound with the contraction lemma.
\end{proof}
The above lemma immediately implies that if $k \ge 4L^2/\epsilon^2$, then
there exist $i_1,\ldots,i_k$ in $I_1$ such that if $\hat{\w} \in
\argmin_{\w : \|\w\| \le 1} \sum_{j=1}^k g(b_{i_j}
\inner{\w,\x_{i_j}})$ then $\hat{\w}$ satisfies \eqref{eqn:regsuff}
and therefore it is an $\epsilon$-accurate solution of
\eqref{eqn:optERM2}. The procedure will therefore perform an exhaustive 
search over all $i_1,\ldots,i_k$ in $[m]$, for each such sequence the
procedure will find  $\hat{\w} \in
\argmin_{\w : \|\w\| \le 1} \sum_{j=1}^k g(b_{i_j}
\inner{\w,\x_{i_j}})$ (in polynomial time). Finally, the procedure
will output the $\hat{\w}$ that minimizes the objective of
\eqref{eqn:optERM2}.
The total runtime of the procedure is therefore 
$\poly(m^k)$. Plugging in the value of $k = \lceil 4L^2/\epsilon^2
\rceil$ and the value of $m$ according to the sample complexity bound
given in \thmref{thm:Rad} we obtain the total runtime of
\[
\poly\left( (L/\epsilon)^{L^2/\epsilon^2}  \right) = 
\poly\left( \exp\left( \tfrac{L^2}{\epsilon^2} \log(L/\epsilon)\right) \right) ~.
\]

\section{Proof of \lemref{lem:sig}} \label{sec:Sigmoid}


In order to approximate $\phi_{\sig}$ with a polynomial, we will use the
technique of \emph{Chebyshev approximation} (cf. \citep{Mason03}). One can
write any continuous function on $[-1,+1]$ as a Chebyshev expansion
$\sum_{n=0}^{\infty}\alpha_n T_{n}(\cdot)$, where each $T_{n}(\cdot)$ is a
particular $n$-th degree polynomial denoted as the $n$-th Chebyshev polynomial
(of the first kind). These polynomials are defined as $T_{0}(x)=1,T_{1}(x)=x$,
and then recursively via $T_{n+1}(x)=2xT_{n}(x)-T_{n-1}(x)$. For any $n$,
$T_{n}(\cdot)$ is bounded in $[-1,+1]$. The coefficients in the Chebyshev
expansion of $\phi_{\sig}$ are equal to
\begin{equation}\label{eq:def_alphan}
\alpha_n = \frac{1+\mathbf{1}(n>0)}{\pi}\int_{x=-1}^{1}\frac{\phi_{\sig}(x)T_{n}(x)}{\sqrt{1-x^2}}dx.
\end{equation}
Truncating the series after some threshold $n=N$ provides an $N$-th degree
polynomial which approximates the original function.

In order to obtain a bound on B, we need to understand the behavior of the
coefficients in the Chebyshev approximation. These are determined in turn by
the behavior of $\alpha_n$ as well as the coefficients of each Chebyshev
polynomial $T_{n}(\cdot)$. The following two lemmas provide the necessary
bounds.

\begin{lemma}\label{lem:a_bound}
  For any $n> 1$, $|\alpha_n|$ in the Chebyshev expansion of $\phi_{\sig}$ on
  $[-1,+1]$ is upper bounded as follows:
\[
|\alpha_n|\leq \frac{1/2L+1/\pi}{(1+\pi/4L)^n}.
\]
Also, we have $|\alpha_0|\leq 1$, $|\alpha_1|\leq 2$.
\end{lemma}

\begin{proof}
  The coefficients $\alpha_n$, $n=1,\ldots$ in the Chebyshev series are given
  explicitly by
\begin{equation}\label{eq:alphadef}
  \alpha_n = \frac{2}{\pi}\int_{x=-1}^{1}\frac{\phi_{\sig}(x)T_{n}(x)}{\sqrt{1-x^2}}dx.
\end{equation}
For $\alpha_0$, the same equality holds with $2/\pi$ replaced by $1/\pi$, so
$\alpha_0$ equals
\[
\frac{1}{\pi}\int_{x=-1}^{1}\frac{\phi_{\sig}(x)}{\sqrt{1-x^2}}dx,
\]
which by definition of $\phi_{\sig}(x)$, is at most
$(1/\pi)\int_{x=-1}^{1}\left(\sqrt{1-x^2}\right)^{-1}dx = 1$.  As for
$\alpha_1$, it equals
\[
\frac{2}{\pi}\int_{x=-1}^{1}\frac{\phi_{\sig}(x)x}{\sqrt{1-x^2}}dx,
\]
whose absolute value is at most
$(2/\pi)\int_{x=-1}^{1}\left(\sqrt{1-x^2}\right)^{-1}dx = 2$.

To evaluate the integral in \eqref{eq:alphadef} for general $n$ and $L$, we 
will need to use some tools from complex analysis. The calculation follows 
\citep{Elliot64}, to which we refer the reader for justification of the steps 
and further details\footnote{We note that such calculations also appear in
  standard textbooks on the subject, but they are usually carried under
  asymptotic assumptions and disregarding coefficients which are important for
  our purposes.}.

On the complex plane, the integral in \eqref{eq:alphadef} can be viewed as a
line integral over $[-1,+1]$. Using properties of Chebyshev polynomials, this
integral can be converted into a more general complex-valued integral over an
arbitrary closed curve $C$ on the complex plane which satisfies certain
regularity conditions:
\begin{equation}\label{eq:alphan_complex}
  \alpha_n = \frac{1}{\pi i}\int_{C}\frac{\phi_{\sig}(z)dz}{\sqrt{z^2-1}(z\pm \sqrt{z^2-1})^n}dz,
\end{equation}
where the sign in $\pm$ is chosen so that $|z\pm \sqrt{z^2-1}|>1$. In
particular, for any parameter $\rho>1$, the set of points $z$ satisfying $|z\pm
\sqrt{z^2-1}|=\rho$ form an ellipse, which grows larger with $\rho$ and with
foci at $z=\pm 1$ and which grows larger with $\rho$. Since we are free to
choose $C$, we choose it as this ellipse while letting $\rho \rightarrow
\infty$.

To understand what happens when $\rho \rightarrow \infty$, we need to
characterize the singularities of $\phi_{\sig}(z)$, namely the points $z$ where
$\phi_{\sig}(z)$ is not well defined. Recalling that $\phi_{\sig}(z) =
(1+e^{-4Lz})^{-1}$, we see that the problematic points are $i(\pi + 2\pi 
k)/4L$ for any $k=\pm 1, \pm 2,\ldots$, where the denominator in 
$\phi_{\sig}(z)$ equals zero. Note that this forms a discrete set of isolated 
points - in other words, $\phi_{\sig}$ is a \emph{meromorphic function}. The 
fact that $\phi_{\sig}$ is 'well behaved' in this sense allows us to perform 
the analysis below.

The behavior of the function at its singularities is defined via the
\emph{residue} of the function at each singularity $c$, which equals
$\lim_{z\rightarrow c} (z-c)\phi_{\sig}(z)$ assuming the limit exists (in that
case, the singularity is called a \emph{simple pole}, otherwise a higher order
limit might be needed). In our case, the residue for the singularity at
$i\pi/4L$ equals
\[
\lim_{z\rightarrow 0} \frac{z}{1+e^{-i \pi -4Lz}} = \lim_{z\rightarrow 0} 
\frac{z}{1-e^{-4Lz}} = \lim_{z\rightarrow 0} \frac{1/4L}{e^{-4Lz}} = 1/4L,
\]
where we used l'H\^{o}pital's rule to calculate the limit. The same residue
also apply to all the other singularities.

For points in the complex plane uniformly bounded away from these
singularities, $|\phi_{\sig}(z)|$ is bounded, and therefore it can be shown
that the integral in \eqref{eq:alphan_complex} will tend to zero as we let $C$
become an arbitrarily large ellipse (not passing too close to any of the
singularities) by taking $\rho\rightarrow \infty$. However, as $\rho$ varies
smoothly, the ellipse does cross over singularity points, and these contribute
to the integral. For meromorphic functions, with a discrete set of isolated
singularities, we can simply sum over all contributions, and it can be shown
(see equation $10$ in \citep{Elliot64} and the subsequent discussion) that
\[
\alpha_n = -2\sum_{k=-\infty}^{\infty}\frac{r_k}{\sqrt{z_k^2-1}\left(z_k\pm
    \sqrt{z_k^2-1}\right)^n},
\]
where $z_k$ is the singularity point $i(\pi + 2\pi k)/4L$ with corresponding 
residue $r_k$. Substituting the results for our chosen function, we have
\[
\alpha_n = \sum_{k=-\infty}^{\infty}\frac{1/4L} {\sqrt{\left(i(\pi+2\pi
      k)/4L\right)^2-1}\left(i(\pi+2\pi k)/4L\pm \sqrt{\left(i(\pi+2\pi
        k)/4L\right)^2-1}\right)^n}.
\]
A routine simplification leads to the following\footnote{On first look, it
  might appear that $\alpha_n$ takes imaginary values for even $n$, due to the
  $i^{n+1}$ factor, despite $\alpha_n$ being equal to a real-valued
  integral. However, it can be shown that $\alpha_n=0$ for even $n$. This
  additional analysis can also be used to slightly tighten our final results in
  terms of constants in the exponent, but it was not included for simplicity.}:
\[
\alpha_n = \sum_{k=-\infty}^{\infty}\frac{1/4L} {i^{n+1}\sqrt{\left((\pi+2\pi
      k)/4L\right)^2+1}\left((\pi+2\pi k)/4L\pm \sqrt{\left((\pi+2\pi
        k)/4L\right)^2+1}\right)^n}.
\]
It can be verified that $\pm$ should be chosen according to $\indct{k\geq
  0}$. Therefore,
\begin{align*}
  |\alpha_n| &= \sum_{k=-\infty}^{\infty}\frac{1/4L}
  {\sqrt{\left((\pi+2\pi k)/4L\right)^2+1}\left(|\pi+2\pi k|/4L+ \sqrt{\left((\pi+2\pi k)/4L\right)^2+1}\right)^n}\\
  & \leq \sum_{k=-\infty}^{\infty}\frac{1/4L} {\left(|\pi+2\pi
      k|1/4L+1\right)^n} \leq \frac{1/4L}{(1+\pi/4L)^n}+
  2\sum_{k=1}^{\infty}\frac{1/4L}
  {\left(1+\pi(1+2k)/4L\right)^n}\\
  &\leq \frac{1/4L}{(1+\pi/4L)^n}+\int_{k=0}^{\infty}\frac{1/2L}
  {\left(1+\pi(1+2k)/4L\right)^n}dk\\
\end{align*}
Solving the integral and simplifying gives us 
\[
|\alpha_n|\leq
\frac{1}{(1+\pi/4L)^n}\left(1/4L+\frac{1+\pi/4L}{\pi(n-1)}\right).
\]
Since $n\geq 2$, the result in the lemma follows.
\end{proof}

\begin{lemma}\label{lem:t_bound}
  For any non-negative integer $n$ and $j=0,1,\ldots,n$, let $t_{n,j}$ be the
  coefficient of $x^j$ in $T_{n}(x)$. Then $t_{n,j}=0$ for any $j$ with a
  different parity than $n$, and for any $j>0$,
\[
|t_{n,j}| \leq \frac{e^{n+j}}{\sqrt{2\pi}}
\]

\end{lemma}
\begin{proof}
  The fact that $t_{n,j}=0$ for $j,n$ with different parities, and
  $|t_{n,0}|\leq 1$ is standard. Using an explicit formula from the literature
  (see \citep{Mason03}, pg. 24), as well as Stirling approximation, we have
  that
\begin{align*}
  &|t_{n,j}| ~=~
  2^{n-(n-j)-1}\frac{n}{n-\frac{n-j}{2}}\binom{n-\frac{n-j}{2}}{\frac{n-j}{2}}
  ~=~ \frac{2^j n}{n+j}\frac{\left(\frac{n+j}{2}\right)!}{\left(\frac{n-j}{2}\right)!j!}\\
  &\leq \frac{2^j n}{j!(n+j)}\left(\frac{n+j}{2}\right)^{j} ~=~
  \frac{n(n+j)^{j}}{(n+j)j!}
  ~\leq~ \frac{n(n+j)^{j}}{(n+j)\sqrt{2\pi j}(j/e)^j} ~=~ \frac{ne^j}{(n+j)\sqrt{2\pi j}}\left(1+\frac{n}{j}\right)^{j}\\
  &\leq \frac{ne^j}{(n+j)\sqrt{2\pi j}}e^n.
\end{align*}
from which the lemma follows.
\end{proof}

We are now in a position to prove a bound on B. As discussed earlier,
$\phi_{\sig}(x)$ in the domain $[-1,+1]$ equals the expansion
$\sum_{n=0}^{\infty}\alpha_{n}T_{x}$. The error resulting from truncating the
Chebyshev expanding at index $N$, for any $x\in [-1,+1]$, equals
\[
\left|\phi_{\sig}(x)-\sum_{n=0}^{N}\alpha_{n}T_{n}(x)\right| =
\left|\sum_{n=N+1}^{\infty}\alpha_{n}T_{n}(x)\right| \leq
\sum_{n=N+1}^{\infty}|\alpha_{n}|,
\]
where in the last transition we used the fact that $|T_{n}(x)|\leq 1$. Using
\lemref{lem:a_bound} and assuming $N>0$, this is at most
\[
\sum_{n=N+1}^{\infty}\frac{1/2L+1/\pi}{(1+\pi/4L)^n}
 = \frac{2+4L/\pi}{\pi(1+\pi/4L)^N}.
\]
In order to achieve an accuracy of less than $\epsilon$ in the approximation, 
we need to equate this to $\epsilon$ and solve for $N$, i.e.
\begin{equation}\label{eq:N_bound}
N= \left\lceil \log_{1+\pi/4L}\left(\frac{2+4L/\pi}{\pi\epsilon}\right) 
\right\rceil
\end{equation}

The series left after truncation is $\sum_{n=0}^{N}\alpha_n T_{n}(x)$, which we
can write as $\sum_{j=0}^{N}\beta_j x^j$. Using \lemref{lem:a_bound} and
\lemref{lem:t_bound}, the absolute value of the coefficient $\beta_j$ for $j>1$
can be upper bounded by
\begin{align*}
  &\sum_{n=j..N, n=j \mod 2}|a_n||t_{n,j}| \leq \sum_{n=j..N, n=j \mod 2}
  \frac{1/2L+1/\pi}{(1+\pi/4L)^n}\frac{e^{n+j}}{\sqrt{2\pi}}\\
  &=\frac{(1/2L+1/\pi) e^j}{\sqrt{2\pi}}\sum_{n=j..N, n=j \mod 2}\left(\frac{e}{1+\pi/4L}\right)^n\\
  &=\frac{(1/2L+1/\pi) e^j}{\sqrt{2\pi}}\left(\frac{e}{1+\pi/4L}\right)^j \sum_{n=0}^{\lfloor \frac{N-j}{2}\rfloor}\left(\frac{e}{1+\pi/4L}\right)^{2n}\\
  &\leq \frac{(1/2L+1/\pi)
    e^j}{\sqrt{2\pi}}\left(\frac{e}{1+\pi/4L}\right)^j
  \frac{(e/(1+\pi/4L))^{N-j+2}-1}{(e/(1+\pi/4L))^2-1}.
\end{align*}
Since we assume $L\geq 3$, we have in particular $e/(1+\pi/4L)>1$, so we can 
upper bound the expression above by dropping the $1$ in the numerator, to get
\begin{align*}
  \frac{1/2L+1/\pi}{\sqrt{2\pi}((e/(1+\pi/4L))^2-1)}\left(\frac{e}{1+\pi/4L}\right)^{N+2}e^{j}.
\end{align*}

The cases $\beta_0,\beta_1$ need to be treated separately, due to the different
form of the bounds on $\alpha_0,\alpha_1$. Repeating a similar analysis (using
the actual values of $t_{n,1},t_{n,0}$ for any $n$), we get
\begin{align*}
&\beta_0 \leq 1+\frac{1}{\pi}+\frac{2L}{\pi^2}\\ &\beta_1 \leq 
2+\frac{3(1+2L/\pi)(4L+\pi)}{2\pi^2}.
\end{align*}

Now that we got a bound on the $\beta_j$, we can plug it into the bound on $B$, and get
\begin{align*}
  &B = \sum_{j=0}^{N}2^j \beta_j^2
  \leq \beta_0^2+2\beta_1^2+\sum_{j=2}^{N}\left(\frac{1/2L+1/\pi}{\sqrt{2\pi}((e/(1+\pi/4L))^2-1)}\right)^2\left(\frac{e}{1+\pi/4L}\right)^{2N+4}(2e^{2})^j\\
  &\leq\beta_0^2+2\beta_1^2+\left(\frac{1/2L+1/\pi}{\sqrt{2\pi}((e/(1+\pi/4L))^2-1)}\right)^2\left(\frac{e}{1+\pi/4L}\right)^{2N+4}\frac{(2e^2)^{N+1}}{e^2-1}\\
  &=\beta_0^2+2\beta_1^2+\frac{2(1/2L+1/\pi)^2e^6}{(e^2-1)2\pi((e/(1+\pi/4L))^2-1)^2(1+\pi/4L)^4}\left(\frac{\sqrt{2}e^2}{1+\pi/4L}\right)^{2N}.
\end{align*}
To make the expression more readable, we use the (rather arbitrary) assumption 
that $L\geq 3$. In that case, by some numerical calculations, it is not 
difficult to show that we can upper bound the above by
\[
2L^4+0.15\left(\frac{\sqrt{2}e^2}{1+\pi/4L}\right)^{2N} \leq 2L^4+0.15(2e^4)^N.
\]
Combining this with \eqref{eq:N_bound}, we get that this is upper bounded by
\[
2L^4+0.15(2e^4)^{\log_{1+\pi/4L}\left(\frac{2+4L/\pi}{\pi\epsilon}\right)+1},
\]
or at most
\[
2L^4+\exp\left(\frac{\log(2e^4)\log\left(\frac{2+4L/\pi}{\pi\epsilon}\right)}{\log(1+\pi/4L)}+3\right).
\]
Using the fact that $\log(1+x)\geq x-x^2$ for $x\geq 0$, and the assumption 
that $L\geq 3$, we can bound the exponent by
\[
\frac{\log(2e^4)\log\left(\frac{2+4L/\pi}{\pi\epsilon}\right)}{\frac{\pi}{4L}\left(1-\frac{\pi}{8L}\right)}+3
\leq 7\log(2L/\epsilon)L+3.
\]
Substituting back, we get the result stated in \lemref{lem:sig}.

\section{The $\phi_{\erf}(\cdot)$ Function} \label{sec:erf}

In this section, we prove a result anaologous to \lemref{lem:sig}, using the
$\phi_{\erf}(\cdot)$ transfer function. In a certain sense, it is stronger,
because we can show that $\phi_{\erf}(\cdot)$ actually belongs to $P_{B}$ for
sufficiently large $B$. However, the resulting bound is worse than
\lemref{lem:sig}, as it depends on $\exp(L^2)$ rather than
$\exp(L)$. However, the proof is much simpler, which helps to illustrate the
technique.

The relevant lemma is the following:
\begin{lemma}\label{lem:erf}
  Let $\phi_{\erf}(\cdot)$ be as defined in \eqref{eq:erfpw}, where for
  simplicity we assume $L\geq 3$. For any $\epsilon>0$, let
\[
B\leq
\frac{1}{4}+2L^2\left(1+3\pi e L^2 e^{4\pi L^2}\right).
\]
Then $\phi_{\erf}(\cdot)\in P_{B}$. 
\end{lemma}

\begin{proof}
  By a standard fact, $\phi_{\erf}(\cdot)$ is equal to its infinite Taylor
  series expansion at any point, and this series equals
\[
\phi_{\erf}(a) =
\frac{1}{2}+\frac{1}{\sqrt{\pi}}\sum_{n=0}^{\infty}\frac{(-1)^n
  (\sqrt{\pi}La)^{2n+1}}{n!(2n+1)}.
\]
Luckily, this is an infinite degree polynomial, and it is only left to
calculate for which values of $B$ does it belong to $P_{B}$. Plugging in the
coefficients in the bound on $B$, we get that
\begin{align*}
  &B\leq \frac{1}{4}+\frac{1}{\pi}\sum_{n=0}^{\infty}
  \frac{(2 \pi L^2)^{2n+1}}{(n!)^2(2n+1)^2 }
  \leq \frac{1}{4}+\frac{1}{\pi}\sum_{n=0}^{\infty}\frac{(2\pi L^2)^{2n+1}}{(n!)^2}\\
  &= \frac{1}{4}+2L^2\left(1+\sum_{n=1}^{\infty}\frac{(2\pi L^2)^{2n}}{(n!)^2}\right)
  \leq \frac{1}{4}+2L^2\left(1+\sum_{n=1}^{\infty}\frac{(2\pi L^2)^{2n}}{(n/e)^{2n}}\right)\\
  &= \frac{1}{4}+2L^2\left(1+\sum_{n=1}^{\infty}\left(\frac{2\pi eL^2}{n}\right)^{2n}\right).
\end{align*}
Thinking of $(2\pi eL^2/n)^{2n}$ as a continuous function of $n$, a simple
derivative exercise shows that it is maximized for $n=2\pi L^2$, with value
$e^{4\pi L^2}$. Therefore, we can upper bound the series in the expression
above as follows:
\begin{align*}
  &\sum_{n=1}^{\infty}\left(\frac{2\pi eL^2}{n}\right)^{2n}
  = \sum_{n=1}^{\lfloor 2\sqrt{2}\pi e L^2 \rfloor}\left(\frac{2\pi eL^2}{ n}\right)^{2n} + \sum_{n=\lceil 2\sqrt{2}\pi eL^2 \rceil}^{\infty} \left(\frac{2\pi eL^2}{ n}\right)^{2n}\\
  &\leq 2\sqrt{2}\pi eL^2e^{4\pi L^2}+\sum_{n=\lceil
    2\sqrt{2}\pi eL^2 \rceil}^{\infty}\left(\frac{1}{2}\right)^{n}\leq
  3\pi e L^2 e^{4\pi L^2}.
\end{align*}
where the last transition is by the assumption that $L\geq 3$. Substituting 
into the bound on $B$, we get the result stated in the lemma.
\end{proof}

\end{document}